\documentclass[twocolumn,10pt]{article}

\usepackage{cite,hyperref}
\usepackage{multirow}
\usepackage{amsmath,amssymb,amsfonts,amsthm}
\usepackage{bbm}
\usepackage{algorithm} 
\usepackage{algpseudocode}
\usepackage{graphicx}
\usepackage{subcaption}
\usepackage{textcomp}
\usepackage{mathtools}
\usepackage{xcolor}
\usepackage{xspace}

\usepackage[
a4paper,
left=1.5cm,
right=1.5cm,
top=1cm,
bottom=2cm,
margin=1.5cm,
headsep=5mm,
footskip=9mm
]{geometry}

\usepackage{dsfont}
  
\newcommand{\bs}{\boldsymbol}
\newcommand{\cl}{\mathcal}
\newcommand{\bb}{\mathbb}
\newcommand{\bbm}{\mathbbm}

\newtheorem{theorem}{Theorem}
\newtheorem{lemma}[theorem]{Lemma}
\newtheorem{proposition}[theorem]{Proposition}

\newtheorem{assumption}{Assumption}

\newcommand{\mathtxt}[1]{%
  \ifmmode
  \mathrm{#1}%
  \else%
  #1\@\xspace%
  \fi%
}

\newcommand{\iid}{\mathtxt{i.i.d.}}

\DeclareMathOperator{\vspan}{span}

\DeclareMathOperator{\diag}{diag}

\DeclareMathOperator{\ortho}{ortho}

\newcommand{\scp}[3][]{#1\langle #2, #3 #1\rangle}

\renewcommand{\leq}{\leqslant}
\renewcommand{\geq}{\geqslant}

\newcommand{\ts}{\textstyle}

\newcommand{\ie}{\emph{i.e.}, }
\newcommand{\eg}{\emph{e.g.}, }

\newcommand{\bftilde}[1]{\tilde{\boldsymbol{#1} \hspace{0.28em}} \hspace{-0.28em}}
\newcommand{\bhat}[1]{\hat{\boldsymbol{#1} \hspace{0.28em}} \hspace{-0.28em}}

\DeclareRobustCommand{\rchi}{{\mathpalette\irchi\relax}}
\newcommand{\irchi}[2]{\raisebox{0.05cm}{$#1\chi$}}

\newcommand{\kernel}{\Gamma}
\newcommand{\bkernel}{\bs\Gamma}
\newcommand{\Ocomp}{\cl O}

\usepackage{xcolor}

\usepackage{titlesec}

\titlespacing{\section}{0pt}{8pt}{5pt}

\renewcommand{\thesubsection}{\thesection.\Alph{subsection}}
\titleformat{\subsection}[runin]
  {\em}
  {\thesubsection.}
  {.5em}
  {}
\titlespacing{\subsection}{0pt}{4pt}{5pt}

\begin{document}

\title{\vspace{-5mm}\huge Random Wavelet Features for Graph Kernel Machines\\[2mm]
\large (extended version with supplementary material)\vspace{-4mm}
\footnote{This work was supported by the Fonds de la Recherche Scientifique - FNRS under Grants n° 40038614 and FNRS T.0160.24. During this research, LJ benefited from an INRIA Chair at the Institute for Advanced Study, Collegium of Lyon, working in the OCKHAM team, at ENS Lyon.}
}

\author{Valentin de Bassompierre, Jean-Charles Delvenne, and Laurent Jacques\\[2mm]
\footnotesize INMA/ICTEAM, UCLouvain, Belgium\vspace{-4mm}}
\date{}

\maketitle


\textbf{Abstract}: Node embeddings map graph vertices into low-dimensional Euclidean spaces while preserving structural information. They are central to tasks such as node classification, link prediction, and signal reconstruction. A key goal is to design node embeddings whose dot products capture meaningful notions of node similarity induced by the graph. Graph kernels offer a principled way to define such similarities, but their direct computation is often prohibitive for large networks.

Inspired by random feature methods for kernel approximation in Euclidean spaces, we introduce randomized spectral node embeddings whose dot products estimate a low-rank approximation of any specific graph kernel. We provide theoretical and empirical results showing that our embeddings achieve more accurate kernel approximations than existing methods, particularly for spectrally localized kernels. These results demonstrate the effectiveness of randomized spectral constructions for scalable and principled graph representation learning.

\section{Introduction}

Kernel machines are a class of machine learning algorithms, designed to handle non-linear problems by implicitly mapping data into a high-dimensional feature space. Given some data space $\cl X$, this is achieved through a \emph{positive semi-definite} (psd) kernel 
$\kernel: (\bs x,\bs y) \in \cl X \times \cl X \mapsto \kernel(\bs x,\bs y)$. This kernel can be interpreted, through the \emph{kernel trick}, as a dot product in a (possibly infinite-dimensional) feature (Hilbert) space $\cl H$, \ie there exist a mapping $\phi: \cl X \to \cl H$ such that $\kernel(\bs x,\bs y) = \scp{\phi(\bs x)}{\phi(\bs y)}_{\cl H}$. In this context, algorithms perform linear computations in this enriched space without explicitly computing the feature map 
$\phi$, thereby enabling more expressive linear models.

While kernels are often introduced for data in continuous domains such as $\cl X = \mathbb{R}^d$, many practical problems involve data supported on graphs, where the underlying relational structure between data samples plays a central role. It is therefore essential to define kernels that respect and exploit this structure. As such, \textit{graph kernels} \cite{smola2003kernels, kondor2002diffusion} provide a way to quantify similarity between pairs of nodes. For a graph $\cl G$ with $N$ nodes, such a kernel can be represented as an $N \times N$ kernel matrix---most often, a function of the graph Laplacian. In practice, however, computing this matrix is often prohibitive, with a typical cost scaling as $\Ocomp(N^3)$, which limits the applicability of many existing graph kernel methods.

In $\mathbb{R}^d$, kernel methods scale poorly on large datasets as computing and storing full kernel matrices is costly. To address this, Rahimi and Recht \cite{rahimi2007random} proposed \textit{random (Fourier) features}. These explicitly map data into a low-dimensional space via a randomized feature map $\bs \phi : \mathbb{R}^d \rightarrow \mathbb{R}^D$ so that, for specific kernels $\kernel$ such as the Gaussian or Laplacian kernels, $\kernel(\bs x, \bs y) = \bb E \scp{\bs \phi(\bs x)}{\bs \phi(\bs y)} \approx \scp{\bs \phi(\bs x)}{\bs \phi(\bs y)}$ for all $\bs x, \bs y \in \bb R^d$. This allows fast linear learning on the transformed points, reducing training and evaluation time as well as memory usage. It is thus natural to ask whether graph kernel computations can be similarly accelerated using random features.

We propose to use tools from Graph Signal Processing (GSP) \cite{shuman2013emerging} for randomized graph kernel approximation. GSP leverages the graph Laplacian, and the derived graph Fourier transform, to extend the definition of frequency, filtering, and smoothness for graph signals while respecting the graph structure. Specifically, we leverage the graph wavelet transform \cite{hammond2011wavelets} on random input signals to construct embeddings that provide scalable approximations of graph kernels.

\noindent\emph{Contributions:} Building on random feature methods for kernel approximation in Euclidean spaces, we propose randomized spectral embeddings for graphs, where the dot product between embeddings estimates a low-rank approximation of a Laplacian-based graph kernel. We provide both theoretical and empirical evidence that these embeddings yield more accurate kernel approximations than existing approaches, especially when the kernel is spectrally localized. Finally, we discuss connections to randomized singular value decomposition (RSVD) \cite{halko2011finding}, highlighting links between our method and classical numerical linear algebra techniques.

\section{Related work}
In \cite{tremblay2014graph}, the graph wavelet transform of random signals is used to accelerate the computation of graph matrix functions. Their use case is slightly different from ours, as they approximate correlation matrices based explicitly on wavelets instead of general graph kernels. In \cite{tremblay2016compressive}, the application of graph wavelet transforms to random signals yields embeddings that approximate the distance matrix of classical spectral clustering. We somehow extend these results to general graph kernels.

Other works have considered random feature methods for graph kernel approximation, see \eg \cite{choromanski2023taming, reid2023general}. Their random features follow a spatial design relying on random walks to assess the similarities between nodes, and produce unbiased kernel approximations. We show that the convergence to said kernel is slow for \emph{bandlimited} kernels which are well localized in the spectral domain and spread out in the spatial domain.

\section{Mathematical tools}

\subsection{Spectral graph theory:} We consider an undirected weighted graph $\mathcal{G}=({\sf V},{\sf E},w)$ made of $N$ vertices (or nodes) ${\sf V}$  identified, up to a correct ordering, with the $N=|{\sf V}|$ integers $[N] := \{1,\ldots,N\}$. These nodes are connected with edges in ${\sf E} \subseteq {\sf V} \times {\sf V}$, with a total of $E=|{\sf E}|$ edges, and $w: {\sf E} \rightarrow \mathbb{R}^+$ is a weight function assigning a positive weight to each edge. Any vector $\bs f \in \bb R^N$ can be seen as a function ${\sf V} \to \bb R$ defined on each node of ${\sf V}$. The weight matrix $\bs W \in \mathbb{R}^{N\times N}$ of a graph has entries $W_{ij}$ equal to $w(i,j)$ if $(i,j) \in {\sf E}$, and to 0 otherwise. It is symmetric for undirected graphs. The degree matrix $\bs D \in \mathbb{R}^{N \times N}$ is diagonal with entries $D_{ii} \coloneqq \sum_{j=1}^N W_{ij}$ equal to the degrees of the nodes. Finally, let the normalized graph Laplacian $\bs L$ be $\bs L \coloneqq \bs I - \bs D^{-1/2} \bs W \bs D^{-1/2}$.

The graph Laplacian $\bs L$ plays a central role in GSP and spectral graph theory. Since $\bs L$ is symmetric, it factorizes as $\bs L = \bs V \bs \Lambda \bs V^\top$, where $\bs V = (\bs v_1, \ldots, \bs v_N)$ is an orthogonal matrix whose columns are the eigenvectors $\bs v_k$ of eigenvalue $\lambda_k \in [0,2]$, \ie $\bs L \bs v_k = \lambda_k \bs v_k$,
and the diagonal matrix $\bs \Lambda = \diag(\lambda_1, \ldots, \lambda_N)$ represents the \emph{spectrum} of $\bs L$ \cite{chung1997spectral}.

The eigenvectors $\bs v_k \in \bb R^N$ define a graph spectral domain; as functions over $\cl G$, they can be interpreted as graph harmonics, discrete analogues of Fourier modes. The eigenvalues act as squared graph frequencies: small (resp. large) $\lambda_k$ correspond to slowly (resp. rapidly) varying modes on the graph. 

For any signal $\bs f \in \mathbb{R}^N$ defined over ${\sf V}$, the \emph{graph Fourier transform} (GFT) projects $\bs f$ onto the eigenvectors of the graph Laplacian, \ie its forward and inverse operators are
\begin{align*}
    \bhat{f} = \cl F[\bs f] := \bs V^\top \bs f, \qquad  \bs f = \cl F^{-1}[\bhat{f}] := \bs V \bhat{f}.
\end{align*}
The representation $\bhat{f}$ expresses $\bs f$ in the graph spectral domain.

Given a filter $g:\bb R \to \bb R$ defined in the spectral domain, the filtering $\bs f^g$ of $\bs f$ by $g$ is defined in the spectral domain as 
\begin{equation*}
\cl F\big[\bs f^g\big]_k := g(\lambda_k)\hat{f}_k,\ k \in [N],\ \text{or}\ \cl F\big[\bs f^g\big] = g(\bs \Lambda) \bhat{f},  
\end{equation*}
with $g(\bs \Lambda) = \diag(g(\lambda_1),\ldots,g(\lambda_N))$. The inverse GFT yields 
\begin{align}
  \label{eq:filter-def}
   \bs f^g = \bs V g(\bs \Lambda) \bhat{f} = g(\bs L)\bs f,\ \text{with}\ g(\bs L)=\bs V g(\bs \Lambda)\bs V^\top.  
\end{align}
Therefore, filtering $\bs f$ amounts to multiplying it by the function of the Laplacian $g(\bs L)$. The graph wavelet transform (GWT)  implements spectral filtering with specific filters $g$ providing localization in both the frequency and spatial domains.

While computing $g(\bs L)$ exactly requires a full eigendecomposition (ED) of $\bs L$, which is prohibitive for large graphs, GWT approximates $g$ by a Chebyshev polynomial. If the Laplacian is sparse, \ie if the number of edges $E = \Ocomp(N)$, the GWT can be efficiently obtained through a series of fast matrix-vector products with powers of the Laplacian matrix \cite{hammond2011wavelets}.

\subsection{Graph kernels:} A graph kernel $\bkernel\in\mathbb{R}^{N\times N}$ is defined through a psd similarity function $\kernel: {\sf V}\times{\sf V}\to\mathbb{R}$ between nodes. Most graph kernels are defined as a function $\bkernel=h(\bs L)$ of the graph Laplacian $\bs L$ \cite{smola2003kernels}. Common examples include:
\begin{align}
  \bkernel &= (\bs I + \sigma^2 \bs L)^{-d}, \tag{$d$-regularized lap.}\\
  \bkernel &= (\bs I - \bs L/\sigma)^2,\ \text{with}\ \sigma \geq 2, \tag{$2$-step random walk}\\
    \bkernel &= \exp(-\sigma^2 \bs L), \tag{diffusion process}\\
    \bkernel &= \cos(\sigma^2 \bs L \pi / 4),\ \text{with}\ |\sigma| \leq 1. \tag{inverse cosine}
\end{align}
If $\bkernel=h(\bs L)$, the spectral theorem provides
\begin{equation}
  \label{eq:SimExact}
\ts \bkernel =\sum_{k=1}^N h(\lambda_k) \bs v_k \bs v_k^\top = \bs V h(\bs \Lambda) \bs V^\top, 
\end{equation}
\ie $h(\lambda_k)$ weights the contribution of each Laplacian eigenvector. The function $h$ thus acts as a spectral filter that emphasizes or suppresses specific frequency components. Since kernels act as smoothing operators, they should preserve low-frequency, smooth components while attenuating high-frequency, rapidly varying ones. Accordingly, $h(\lambda)$ should decrease with $\lambda$ \cite{smola2003kernels}. For instance, spectral clustering uses a low-pass filter that retains only the smoothest eigenmodes~\cite{von2007tutorial}.

\section{Proposed Method}

Our approach provides a random feature map
$\bs \phi: i \in {\sf V} \mapsto \bs \phi_i = \bs \Phi \bs \delta_i \in \bb R^K$, with $\bs \Phi$ a $K \times N$ matrix with columns $\{\bs \phi_i\}_{i=1}^N$, and $\bs \delta_i \in \bb R^N$ equal to 1 on the $i$-th node and 0 elsewhere, such that the dot-products $\scp{\bs \phi_i}{\bs \phi_j} = \tilde{\kernel}_{ij}$ define a rank-$K$ approximation $\tilde{\bkernel}$ of the graph kernel $\bkernel := h(\bs L)$. We assume that this kernel is associated with some known positive, decreasing function $h: [0,2] \to \bb R_+$, with $h(0) = 1$. As a reference, the optimal rank-$K$ approximation of $\bkernel$ is given by its truncated spectral
decomposition,
\begin{equation}
\ts \bkernel^{(K)} = \sum_{k=1}^K h(\lambda_k)\,\bs v_k \bs v_k^\top = \bs V_{:K} \bs V_{:K}^\top \bkernel, \label{eq:bestK}
\end{equation}
which can also be viewed as its projection onto the subspace spanned by $\bs V_{:K} := (\bs v_1, \ldots, \bs v_K)$ (i.e., the $K$ smoothest graph harmonics).
We show that $\tilde{\bkernel}$ attains an approximation quality
comparable to this optimum,
$\|\bkernel - \tilde{\bkernel}\| \approx \|\bkernel - \bkernel^{(K)}\| = h(\lambda_{K+1})$,
while bypassing the explicit ED of the Laplacian.

The algorithm building $\bs \Phi$ follows a two-step procedure: we first estimate $\vspan \bs V_{:K}$, then $\bs \Phi$ is constructed from this estimated subspace and the known filter $h$ (see Alg.~\ref{alg:fix}).

\subsection{{[Part\hspace{2pt}1]\hspace{2pt}Range finding}:} \label{sec:RangeFinding}

We first estimate $\vspan \bs V_{:K}$, \ie we find an orthogonal $\bs Q$ such that $\vspan \bs Q \approx \vspan \bs V_{:K}$.

We proceed by filtering $K$ random signals. For simplicity, we take Gaussian random signals $\bs G \in \mathbb{R}^{N \times K}$ (\ie $G_{ij} \sim_\iid \mathcal{N}(0,1)$), but one may also consider using structured (Fourier or Hadamard-based) random matrices as proposed in \cite{halko2011finding,yuOrthogonalRandomFeatures2016}. From \eqref{eq:filter-def}, filtering a signal $\bs g \in \bb R^N$ with an ideal low-pass filter $\rchi_{K}:=\bbm 1_{[0, \lambda_K]}$ yields a vector $\bs b = \rchi_{K}(\bs L) \bs g$ lying exactly in the relevant subspace, \ie $\bs b \in \vspan(\bs V_{:K})$. Further, with probability 1, filtering each column of $\bs G$ with $\rchi_{K}$ yields a basis $\bs B = \rchi_{K}(\bs L) \bs G$ for that subspace. The (Gram-Schmidt) orthogonalization $\bs Q^* = \ortho(\bs B)$ thus provides an exact solution to our problem.

In practice, constructing the ideal low-pass filter is infeasible without computing the first $K$ eigenvectors of $\bs L$ (which would defeat the purpose of the range finding). We thus replace the ideal filter with a polynomial approximation\footnote{In this paper, we consider $p_{\chi}$ to be the Jackson-Chebyshev polynomial approximation to $\rchi_{K}$ \cite{di2016efficient}. This polynomial has a less sharp cut-off than Chebyshev polynomials but doesn't present Gibbs oscillation phenomenon.} $p_\chi$, and obtain 
\begin{equation}
  \label{eq:approx-Q}
    \bs Q = \ortho(p_{\chi}(\bs L) \bs G).
\end{equation}
Replacing $\rchi_K$ with $p_\chi$ means the filtered vectors do not lie exactly in $\vspan \bs V_{:K}$. However, if the polynomial approximation is sufficiently accurate, we expect the filtered vectors to lie very close to that subspace. We characterize in Sec.~\ref{sec:kerApprox} the effect this has on the approximation error $\| \bkernel - \tilde{\bkernel} \|$.

Notice that to construct $p_\chi$ one needs access to the value of $\lambda_K$. We use \cite[Alg. 1]{puyRandomSamplingBandlimited2018}, which finds an estimate $\tilde{\lambda}_K$ by dichotomy on the interval $[0,2]$---estimating the eigenvalue count in $[0,\lambda^{(i)}]$ at each iteration $i$. Interestingly, the eigenvalue count estimation is performed by filtering (a very small number of) random Gaussian signals with the same polynomial low-pass filters $p_{\chi_{i}} \approx \bbm 1_{[0, \lambda^{(i)}]}$ as used in our scheme.

\subsection{{[Part\hspace{2pt}2]\hspace{2pt}Kernel approximation}:}

The matrix $\bs Q$ obtained in Sec.~\ref{sec:RangeFinding} can be viewed alternatively as an orthonormal basis for a subspace close to that spanned by the $K$ foremost eigenvectors of $\bs L$, or simply as a set of $K$ signals living on the graph. With this second view, it is natural to think that we can filter these signals to yield relevant embeddings for our graph. In particular, defining the linear node embedding $\bs \Phi: i \in {\sf V} \mapsto \bs \phi_i = \bs \Phi \bs \delta_i \in \bb R^K$ with
\begin{equation}
  \label{eq:Phi-def}
  \ts \bs \Phi = (\bs \phi_1, \ldots, \bs \phi_N) = (h^{\frac{1}{2}}(\bs L) \bs Q)^\top,
\end{equation}
then the collection of dot-products $(\tilde{\kernel}_{ij} = \scp{\bs \phi_i}{\bs \phi_j})_{i,j=1}^N$ of the embedding of all node pairs yields a rank-$K$ approximation $\tilde{\bkernel}$ of the kernel $\bkernel := h(\bs L)$. To show this, let us first consider the ideal case where $\vspan \bs Q = \vspan \bs V_{:K}$, \ie where $\bs Q \bs Q^\top = \bs V_{:K} \bs V_{:K}^\top$ can be interpreted as the projection onto $\vspan \bs V_{:K}$, and $\bs I - \bs Q \bs Q^\top$ as the projection on its orthogonal complement, the span of $\bs V_{K:} = (\bs v_{K+1}, \ldots, \bs v_N)$. Then,
\begin{align}
  \tilde{\kernel}_{ij}&\ts = \scp{\bs \phi_i}{\bs \phi_j} = \bs \delta_i^\top h^{\frac{1}{2}}(\bs L)\bs {QQ}^\top h^{\frac{1}{2}}(\bs L) \bs \delta_j \nonumber \\
  &\ts =  \bs \delta_i^\top \bs V h^{\frac{1}{2}}(\bs \Lambda) \bs V^\top \bs V_{:K} \bs V_{:K}^\top \bs V h^{\frac{1}{2}}(\bs \Lambda) \bs V^\top \bs \delta_j \nonumber \\
  &\ts =  \bs \delta_i^\top \bs V_{:K} h(\bs \Lambda_{:K}) \bs V_{:K}^\top \bs \delta_j. \nonumber 
\end{align}
with $\bs \Lambda_{:K} = \diag(\lambda_1, \ldots, \lambda_K)$. From \eqref{eq:SimExact}, we thus get
\begin{equation*}
\ts \tilde{\bkernel} = \bs \Phi^\top \bs \Phi = \sum_{k=1}^K h(\lambda_k)\,\bs v_k \bs v_k^\top.
\end{equation*}
The matrix $\tilde{\bkernel}$ \emph{is} thus the best rank-$K$ approximation $\bkernel^{(K)}$ of $\bkernel = \bs V h(\bs \Lambda) \bs V^\top$, or its $k$-truncated SVD version --- see (\ref{eq:bestK}).

Our approach only estimates this ideal case through~\eqref{eq:approx-Q} of Part~I. In this case, $\bs \Phi$ in \eqref{eq:Phi-def} yields a kernel $\tilde{\bkernel}$ that is now an approximation to $\bkernel^{(K)}$: the projector $\bs{QQ}^\top$ does not project exactly onto $\vspan \bs V_{:K}$. As a consequence, some less important graph harmonics diffuse into the kernel approximation, to the detriment of some of the more important graph harmonics.

Moreover, an additional error arises in estimating $\bs \Phi$ in \eqref{eq:Phi-def}: as with $\rchi_K$ in \eqref{eq:approx-Q}, the filter $h^{\frac{1}{2}}$ is replaced by a polynomial approximation $p_h$ allowing for the fast estimation~of
\begin{equation}
    \bs \Phi = (p_h(\bs L) \bs Q)^\top \label{eq:PhiDef}
\end{equation}
with the GWT. A more precise characterization of both errors is given in Sec. \ref{sec:PolApprox} and Sec. \ref{sec:kerApprox}.

\subsection{Connection with Randomized SVD and oversampling:} \label{sec:ConnectRSVD}

Part 1 of Alg. \ref{alg:fix} bears strong resemblance with the range-finding part of the RSVD algorithm developed in \cite{halko2011finding}. In RSVD, the range of the $K$ foremost left singular vectors of a matrix $\bs A$ is estimated with
\begin{align*}
    \bs Q = \ortho(\bs A \bs G).
\end{align*}
In the graph kernel case, $\bs A$ corresponds to $h(\bs L)$. The main difference with (\ref{eq:approx-Q}) is that our algorithm makes use of the knowledge of the structure of $\bs L$ and $h(\bs L)$ to design a filter $p_{\chi}(\bs L)$ better suited than $h(\bs L)$ to the range finding task.

\noindent\emph{Oversampling strategy:} In both algorithms (RSVD and ours), the range finding part is inexact and propagates errors. RSVD mitigates this by oversampling the Gaussian matrix $\bs G$; given $r \ll K$ (\eg $r = 15$), it draws $K + r > K$ Gaussian random vectors in $\bs G \in \bb R^{N \times (K+r)}$. This increases the likelihood that the computed subspace aligns well with the one spanned by the top 
$K$ singular vectors of $\bs A$. We adopt the same strategy in our algorithm.

\begin{algorithm}[!t]\small
\caption{Randomized Low-Rank Kernel Approximation}\label{alg:fix}
\begin{algorithmic}[1]
  
  \Require Graph $\mathcal{G}=(\bs V,\bs E,w)$, Laplacian $\bs L \in \mathbb{R}^{N \times N}$, kernel function $h: [0,2] \rightarrow \mathbb{R}^+$, parameter $K$, parameter $M$, over-sampling parameter $r$
  
\Ensure Node embedding $\bs \Phi$ s.t. $\bs \Phi^\top  \bs \Phi = \tilde{\bkernel} \approx \bkernel = h(\bs L)$.

\Statex \hspace{-7.2mm} \textsc{Part 1: Range finding} 

\State Find $\lambda_K$ using Alg.~1 from \cite{puyRandomSamplingBandlimited2018}.
\State Compute a degree $M$ polynomial approximation $p_{\chi}$ of $\rchi_{K}$
\State Generate $\bs G \in \mathbb{R}^{N \times (K+r)}$ with $G_{ij} \sim_\iid \mathcal{N}(0, 1)\; \forall i, j$
\State Compute $\bs Q = \operatorname{ortho}(p_{\chi}(\bs L)\bs G)$

\Statex \hspace{-7.2mm} \textsc{Part 2: Embedding computation}
\State Compute a degree $M$ polynomial approximation $p_h$ of $h^{\frac{1}{2}}$
\State Compute $\bs \Phi^\top = (\bs \phi_1, \ldots, \bs \phi_N)^\top = p_h(\bs L)\bs Q$
\end{algorithmic}
\end{algorithm}

\subsection{Complexity analysis:}
\label{sec:Complexity}
In Alg.~\ref{alg:fix}, lines~1, 4, and~6 dominate the computational and memory costs. Filtering one signal with a polynomial of order~$M$ requires $M$ matrix-vector products with $\bs L$, yielding a time complexity of $\mathcal{O}(ME)$ and memory $\mathcal{O}(N)$. Hence, \cite[Alg. 1]{puyRandomSamplingBandlimited2018} runs in $\mathcal{O}(ME \log N \log(\epsilon^{-1}))$ time and $\mathcal{O}(N \log N)$ memory\footnote{The algorithm filters $\mathcal{O}(\log N)$ random signals $\mathcal{O}(\log(\epsilon^{-1}))$ times.}, where $\epsilon$ denotes the tolerance on the estimation error $\lvert \tilde{\lambda}_K - \lambda_K \rvert$.

Line~4 costs $\mathcal{O}(MEK)$ to filter $\bs G$ and $\mathcal{O}(NK^2)$ for Gram--Schmidt, with memory $\mathcal{O}(NK)$ to store $\bs G$ and $\bs Q$. Similarly, line~6 requires $\mathcal{O}(MEK)$ time and $\mathcal{O}(NK)$ memory to filter $\bs Q$ and store $\bs \Phi$.

Overall, assuming a fixed tolerance~$\epsilon$ and $K = \Omega(\log N)$, the total time complexity of Alg.~\ref{alg:fix} is $\mathcal{O}(MEK + NK^2)$, and the memory complexity is $\mathcal{O}(NK)$. If the matrix $\tilde{\bkernel}$ is explicitly formed as $\tilde{\bkernel} = \bs{\Phi}^\top\bs{\Phi}$, an additional $\mathcal{O}(N^2K)$ time complexity and $\mathcal{O}(N^2)$ memory cost are incurred. In practice, explicitly forming $\tilde{\bkernel}$ is rarely necessary, as matrix--vector products can be efficiently computed as $\tilde{\bkernel}\bs x = \bs{\Phi}^\top(\bs{\Phi} \bs x)$\footnote{This further highlights the advantage of low-rank approximations, reducing the cost of matrix--vector products from $\mathcal{O}(N^2)$ to $\mathcal{O}(NK)$.}.

\subsection{Polynomial approximation error characterization:}
\label{sec:PolApprox}

Let $p_{\chi}$ and $p_h$ be Chebyshev polynomials of degree $M$, approximating $\rchi_K$ and $h^{\frac{1}{2}}$ respectively. We define
\begin{align}
    \epsilon_{\chi,M} &:= \max_{k > K} |\rchi_K(\lambda_k) - p_{\chi}(\lambda_k)| = \max_{k > K} |p_{\chi}(\lambda_k)|, \label{eq:EpsChi} \\
    \epsilon_{h,M} &:= \max_{k \in [N]} |h^{\frac{1}{2}}(\lambda_k) - p_{h}(\lambda_k)|. \label{eq:EpsH}
\end{align}
\begin{assumption}
    \label{ass:pchi}
    Let $\epsilon_{\chi,M}$ be defined as in (\ref{eq:EpsChi}). We suppose
    \begin{align*}
        \lambda_K \neq \lambda_{K+1} \;\; \text{and} \;\; 1/2 \leq p_{\chi}(\lambda_k) \leq 1 + \epsilon_{\chi,M} \; \forall k \leq K.
    \end{align*}
\end{assumption}
From the Weierstrass theorem and since the spectrum of $\bs L$ is discrete, $\epsilon_{\chi,M}$ and $\epsilon_{h,M}$ vanish for large $M$. However, we postpone the analysis of their decay rate to a future study.

\subsection{Kernel approximation error:}
\label{sec:kerApprox}
Given $\bs \Phi$ from (\ref{eq:PhiDef}) and  $\tilde{\bkernel} := \bs \Phi^\top \bs \Phi$, we can bound the error $\cl E := \| \bkernel - \tilde{\bkernel}\|$ as
\begin{align*}
  \cl E& = \| h(\bs L) - p_h(\bs L) \bs{QQ}^\top p_h(\bs L) \|\\
  &\leq \| h^{\frac{1}{2}}(\bs L) (\bs I -\bs{QQ}^\top) h^{\frac{1}{2}}(\bs L) \| \\
       &\quad + \| h^{\frac{1}{2}}(\bs L) \bs{QQ}^\top h^{\frac{1}{2}}(\bs L) - p_h(\bs L) \bs{QQ}^\top p_h(\bs L)\|\\
       &= \| (\bs I -\bs{QQ}^\top) h^{\frac{1}{2}}(\bs L) \|^2 \\
        &\quad + \| h^{\frac{1}{2}}(\bs L) \bs{QQ}^\top h^{\frac{1}{2}}(\bs L) - p_h(\bs L) \bs{QQ}^\top p_h(\bs L) \|\\
    &:= \cl E_R^2 + \cl E_P.
\end{align*}
\begin{proposition}
    \label{prop:main}
    Under Assumption \ref{ass:pchi}, and with all previous notations holding, we have
    \begin{equation}
        \cl E_P \leq 2\epsilon_{h,M} + \epsilon_{h,M}^2 \label{eq:EPBound}
    \end{equation}
    and
    \begin{multline}
        \ts \bb E \cl E_R \leq h^{\frac{1}{2}}(\lambda_{K+1}) +  2\sqrt{\frac{K}{r-1}} p_{\chi}(\lambda_{K+1}) \\
        \ts + 2e\frac{\sqrt{K + r}}{r} \big(\sum\nolimits_{j>K} p_{\chi}(\lambda_{j})^2\big)^{1/2} \label{eq:ERBound}
    \end{multline}
\end{proposition}

The error $\cl E_P$ is due to replacing the filter $h^\frac{1}{2}$ with it's polynomial approximation $p_h$, while the error $\cl E_R$ is due to the range finding. The last two terms in the bound (\ref{eq:ERBound}) indicate the difference with respect to the error of the optimal rank-$K$ projector $\| (\bs I - \bs V_{:K} \bs V_{:K}^\top) h^{\frac{1}{2}}(\bs L) \| = h^{\frac{1}{2}}(\lambda_{K+1})$.

\medskip

\textit{Proof of (\ref{eq:EPBound}).}
Given the definition of $\epsilon_{h,M}$ in (\ref{eq:EpsH}) and that $|h(\lambda)| \leq 1$, we have
\begin{align*}
    \pushQED{\qed}
    \cl E_P &= \| h^{\frac{1}{2}}(\bs L) \bs{QQ}^\top (h^{\frac{1}{2}}(\bs L) - p_h(\bs L)) \\
    & \qquad + (h^{\frac{1}{2}}(\bs L) - p_h(\bs L)) \bs{QQ}^\top p_h(\bs L) \| \\
    &\leq \| h^{\frac{1}{2}}(\bs L) - p_h(\bs L) \| (\|h^{\frac{1}{2}}(\bs L) \| + \| p_h(\bs L) \|) \\
    &\leq \epsilon_{h,M}(1+(1 + \epsilon_{h,M})).\qedhere
    \popQED
\end{align*}
Proving (\ref{eq:ERBound}) is trickier, and we defer the proof to Appendix~\ref{appendix}.

\section{Experimental results}
\begin{figure*}[t]
    \centering
    \begin{subfigure}[c]{0.32\linewidth}
        \includegraphics[width=\linewidth, trim={0cm 0.2cm 0cm 0.2cm}, clip]{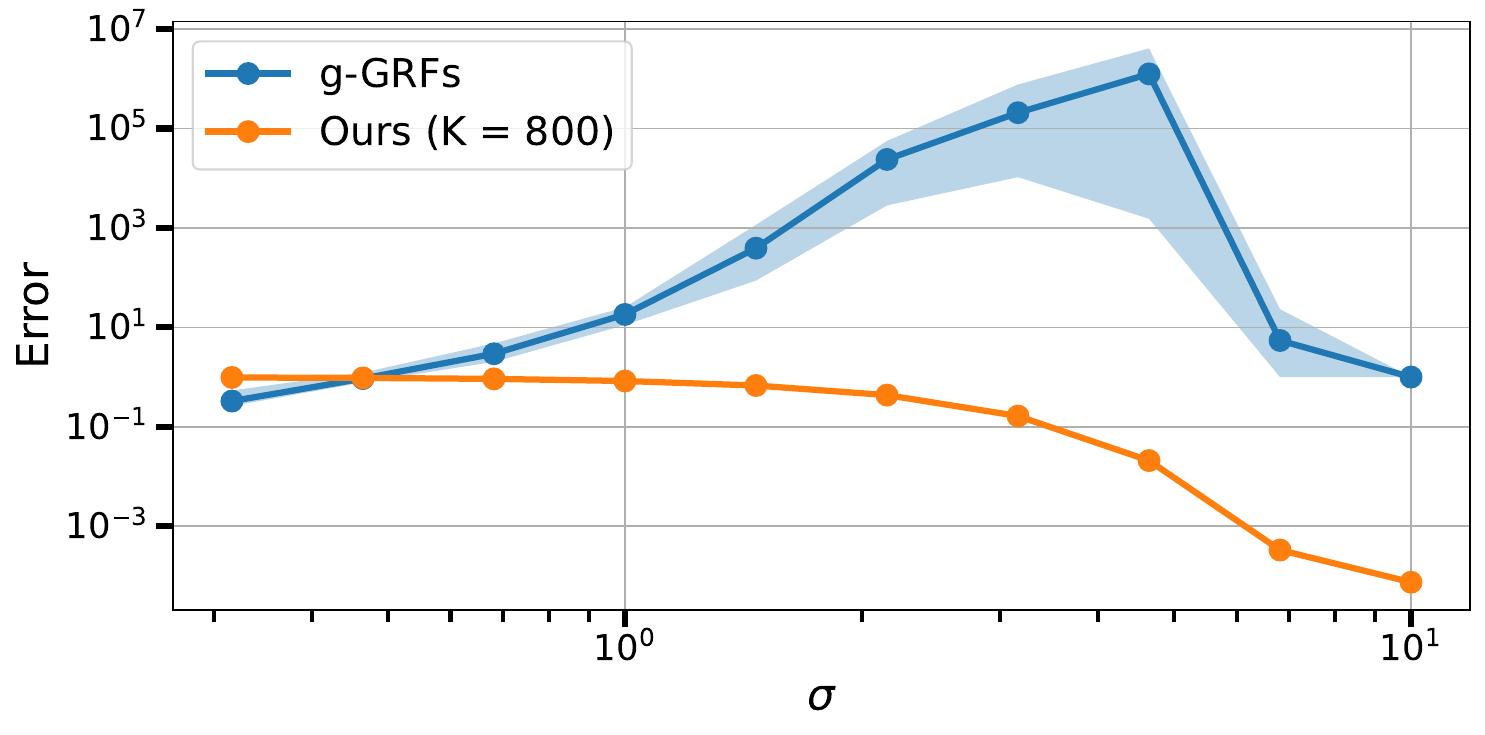}
        \caption{}
        \label{fig:errs}
    \end{subfigure}
    \begin{subfigure}[c]{0.32\linewidth}
        \includegraphics[width=\linewidth, trim={0cm 0.2cm 0cm 0.2cm}, clip]{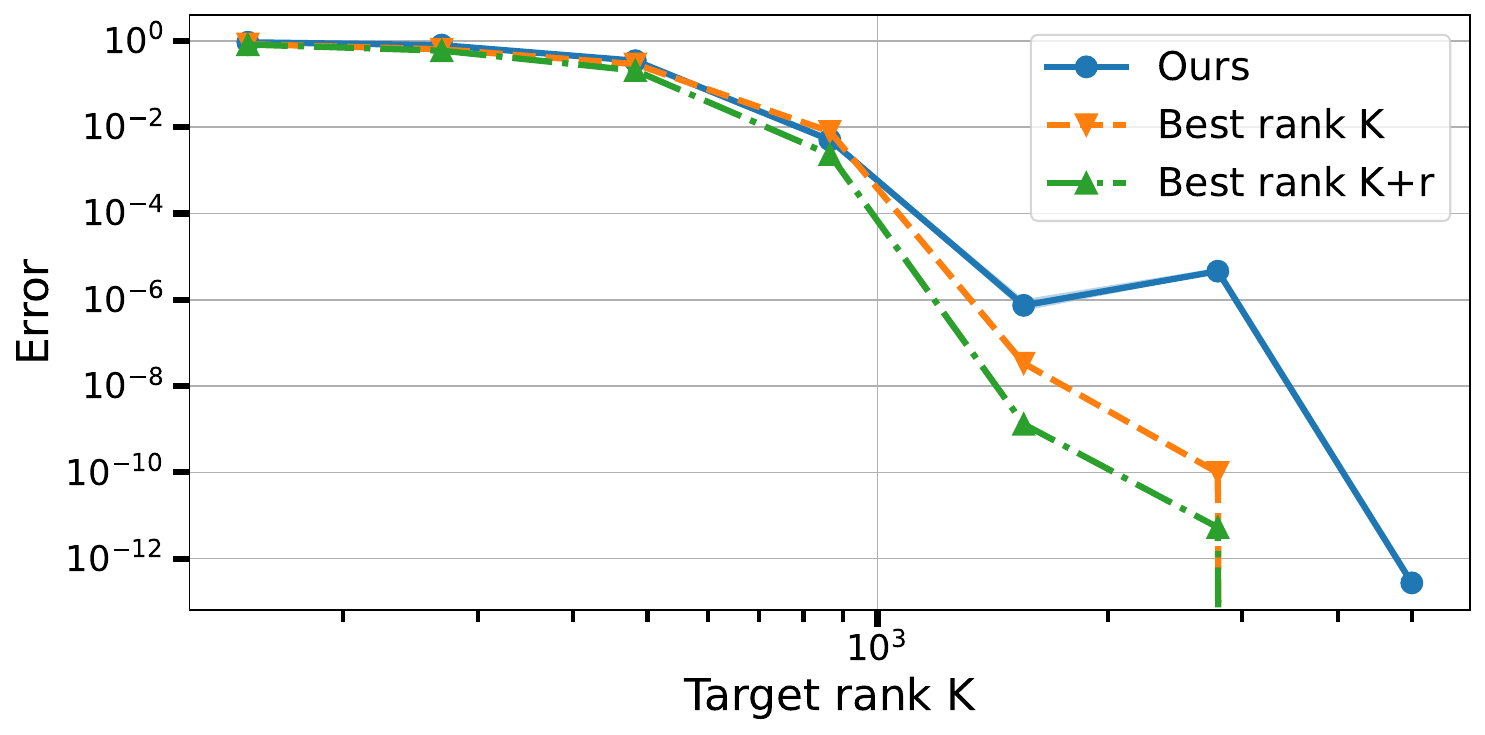}
        \caption{}
        \label{fig:KExpSR}
    \end{subfigure}
    \begin{subfigure}[c]{0.32\linewidth}
        \includegraphics[width=\linewidth, trim={0cm 0.2cm 0cm 0.2cm}, clip]{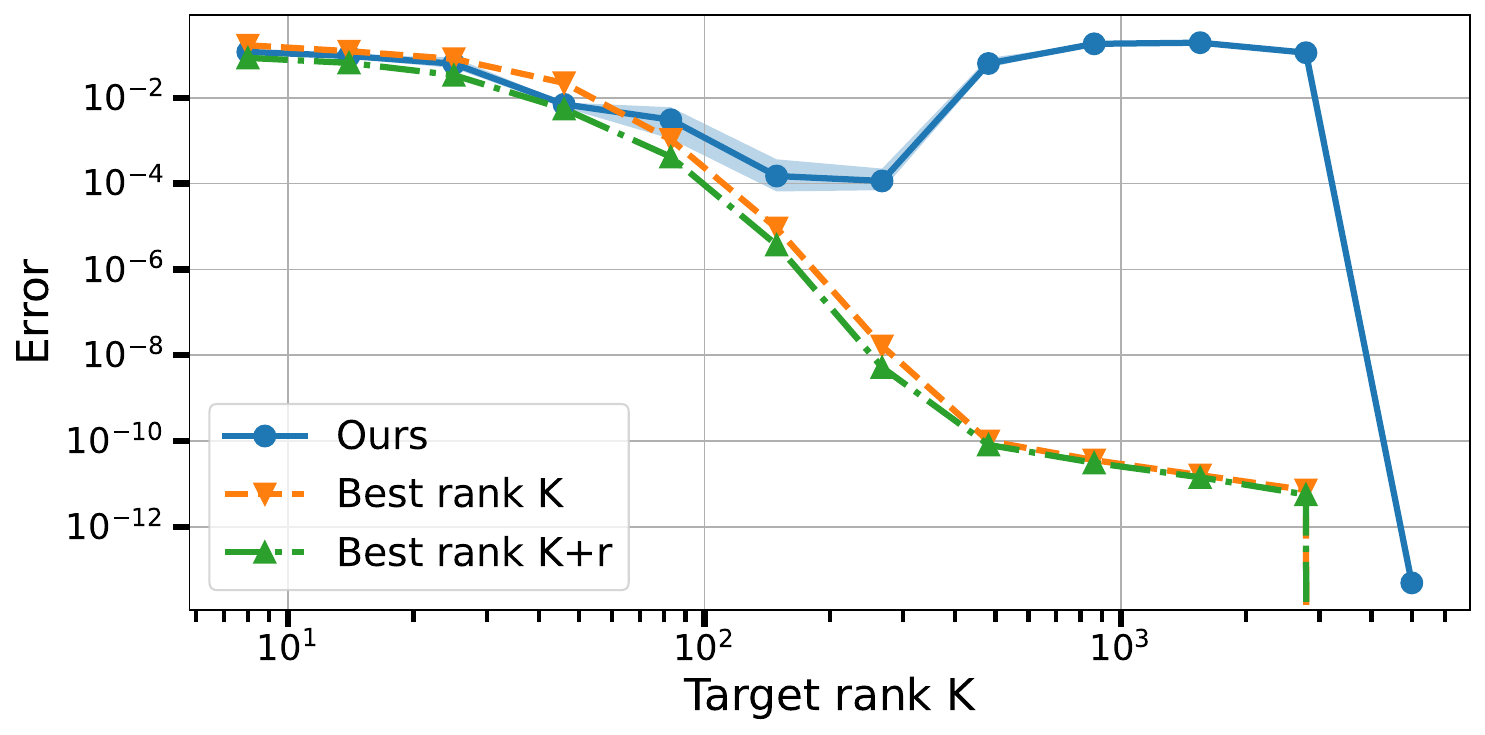}
        \caption{}
        \label{fig:KExpC}
    \end{subfigure}
    \caption{Average relative spectral norm of the error w.r.t. the ground truth kernel, $\| \bkernel - \tilde{\bkernel} \| / \| \bkernel \|$. Shaded areas represent min. and max. values, over 5 trials. Left (a): As a function of $\sigma$, and compared to g-GRFs. Center and right (b-c): As a function of the target rank $K$, for $\sigma=5$, and compared to the best rank-$K$ and rank-$(K+r)$ approximations. Center (b): Swiss-Roll graph. Right (c): Community graph.}
    \label{fig:main}
\end{figure*}
\subsection{Studying kernels with different bandwidths:}
We compare our method to the \emph{general graph random features} (g-GRFs) method\footnote{We use stopping probability $p_{\rm halt} = 0.1$, and number of random walkers $m=8$, which are standard parameters for g-GRFs.} \cite{reid2023general}. As a first experiment, we compare the accuracy of the kernel approximates on a randomly generated Swiss-Roll graph with 5\,000 vertices. Both methods are tested on 10 diffusion kernels $\bkernel = \operatorname{exp}(-\sigma^2 \bs L)$, which controls the bandwidth of the ground-truth kernels (larger $\sigma$ yields stronger spectral localization). Fig. \ref{fig:errs} shows the relative approximation error in these settings\footnote{The computation time is not influenced by $\sigma$. For this experiment, g-GRFs execute on average in $13.57\pm 0.56$ s, and our method in $11.88\pm 0.52$ s.}, for target rank $K=800$ and oversampling parameter $r=K/10=80$. In all experiments, we use $M=30$ for $p_h$ and $M=60$ for $p_\chi$.

Fig.~\ref{fig:errs} highlights that our method is particularly effective for spectrally localized kernels with narrow bandwidths, \ie kernels for which only a small number of graph Laplacian eigenvectors contribute significantly, and a low-rank approximation is appropriate. Such narrow spectral bandwidth implies spatially widespread interactions, highlighting our methods ability to capture global geometric properties. In contrast, g-GRFs are better suited to kernels with wide spectral bandwidths, corresponding to spatially localized behavior, and struggle to capture long-range dependencies. Consequently, our proposed approach complements g-GRFs by effectively approximating spectrally localized kernels.

\subsection{Effect of target rank K:}

We now examine the influence of the parameter $K$, for a randomly generated Swiss-Roll graph with 5\,000 vertices and the diffusion kernel with $\sigma = 5$. We choose $r=\max(K/10,15)$ as oversampling parameter. We compare the results of our method with that of the best rank-$K$ approximation, $\bkernel^{(K)}$ --- for reminder, our method is constructed to achieve a similar error. Since our method effectively yields a rank-$(K+r)$ approximation, we also compare it to the best rank-$(K+r)$ approximation, $\bkernel^{(K+r)}$, which yields a lower-bound on the achievable error. Fig.~\ref{fig:KExpSR} shows the average relative approximation error for these settings.

As expected, the error decreases as the approximation rank increases. We see that our algorithm performs well on the Swiss-Roll graph, closely matching the best rank-$K$ approximation error up to $K \approx 1\,000$. For values of $K$ greater than 1\,000, our method is not as accurate as the optimal rank-$K$ approximation. At this stage, the error stabilizes around $10^{-6}$ (which is perfectly acceptable for most applications). A possible explanation is that small errors in the Gram-Schmidt orthogonalization (which increase with $K$) affect the final error. The error drops again to machine level precision when $K+r \geq N$, where $\bs Q = \bs I$ and the range finding is perfect.

We now repeat the experiment but for a randomly generated Community graph with 5\,000 nodes and 8 strongly connected communities weakly interconnected to each other. Fig.~\ref{fig:KExpC} shows the results for this graph.

As one can see, our method fails to match the approximation quality of the best rank-$K$ approximation for values greater than $K\approx 80$. Moreover, the error increases back to $\approx 10^{-1}$ when reaching $K \approx 500$. We surmise that the explanation lies in the structural difference between the Swiss-Roll and Community graphs. For the Swiss-Roll graph, the eigenvalues of the Laplacian are evenly spread out, while the first few eigenvalues of the Community graph are well separated from the others. After $K \approx 500$, the eigenvalues are very densely packed, such that the eigenvalue count and the estimation of $\lambda_K$ become inaccurate. The spectral cut-off applied with $p_\chi(\bs L)$ is not sharp enough in these densely packed areas, such that a significant amount of undesirable frequencies contaminate the most important ones. This reveals that the performance of the algorithm depends strongly on the eigenvalue distribution and the chosen approximation rank.

\subsection{Experimental complexity:}

We also time our experiments on the Swiss-Roll graph, for different values of $N$ and $K$. In practice, we observe that the two most time consuming steps are the estimation of $\lambda_K$ and the filtering (the computation of $p_\chi(\bs L) \bs G$ and of $p_h(\bs L) \bs Q$). Fig.~\ref{fig:Times} shows the time taken by both these steps\footnote{Computing $p_\chi(\bs L) \bs G$ takes roughly twice as much time as computing $p_h(\bs L) \bs Q$, since we chose polynomials of degree $60$ and $30$ respectively.} and compares them with the asymptotic complexities derived in Sec. \ref{sec:Complexity}. It also compares the total time taken by our algorithm with the time taken to compute the ground truth kernel $\bkernel$ and the best rank-$K$ approximation $\bkernel^{(K)}$ through explicit eigendecomposition.
\begin{figure}[h]
    \centering
    \begin{subfigure}[c]{0.49\linewidth}
        \includegraphics[width=\linewidth, trim={0cm 0cm 0cm 0.2cm}, clip]{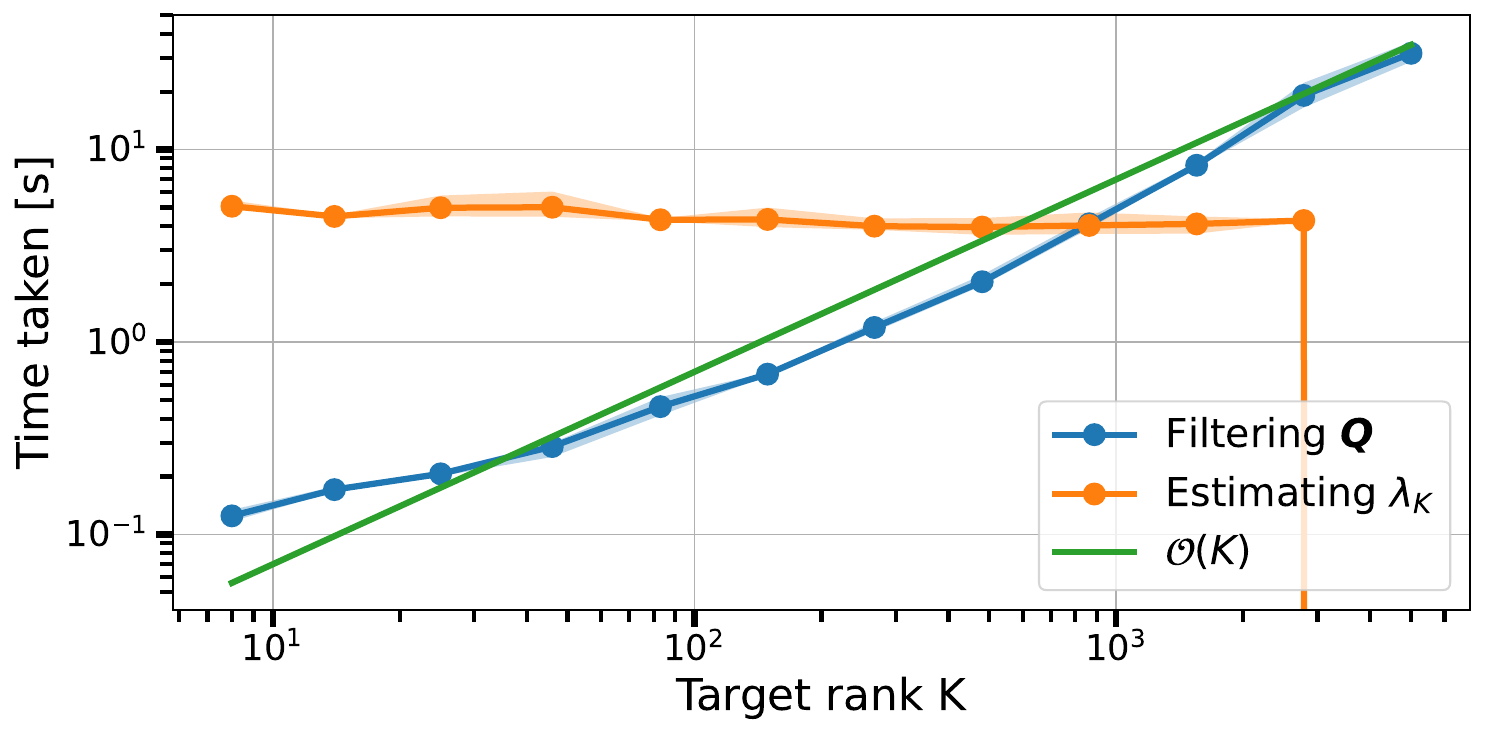}
    \end{subfigure}
    \begin{subfigure}[c]{0.49\linewidth}
        \includegraphics[width=\linewidth, trim={0cm 0cm 0cm 0.2cm}, clip]{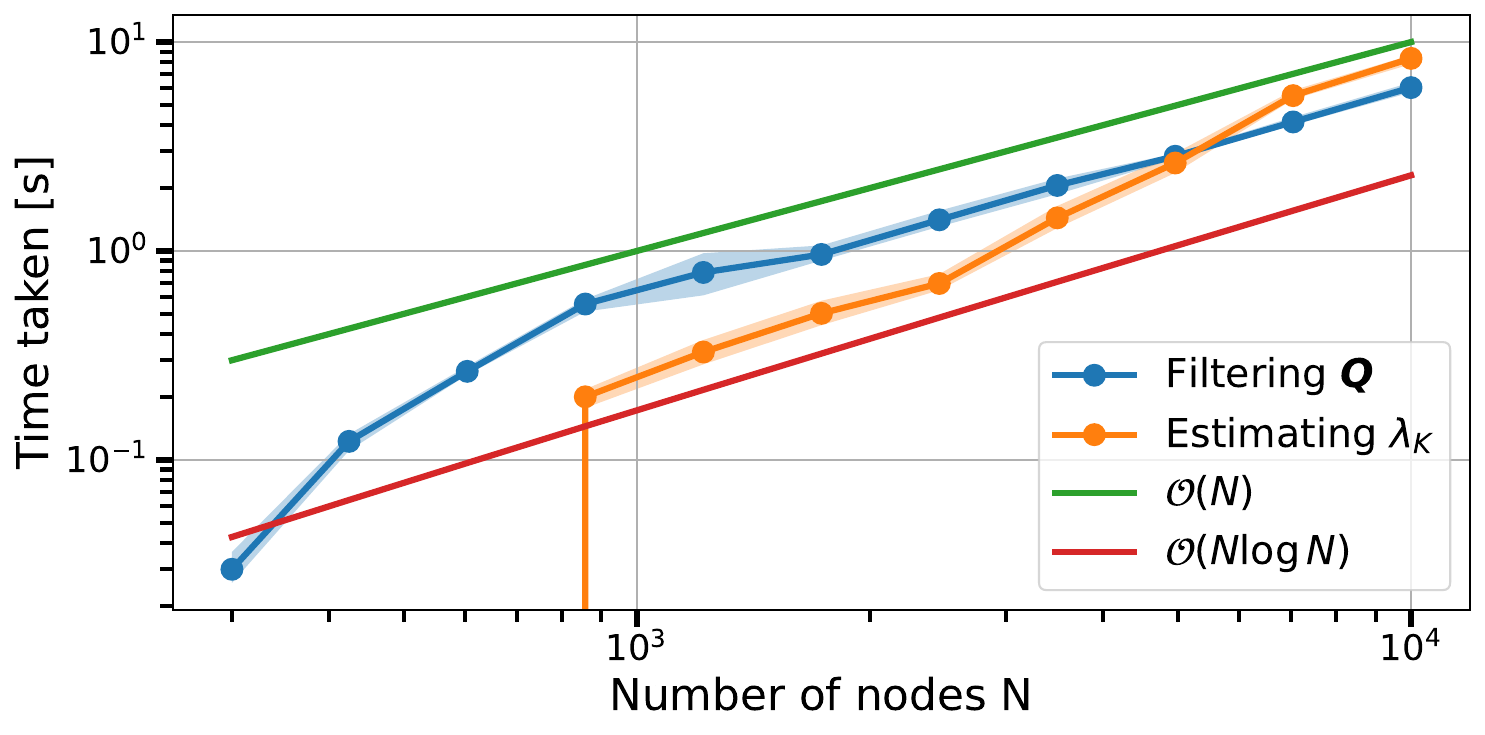}
    \end{subfigure}
    \begin{subfigure}[c]{0.49\linewidth}
        \includegraphics[width=\linewidth, trim={0cm -0.5cm 0cm 0cm}, clip]{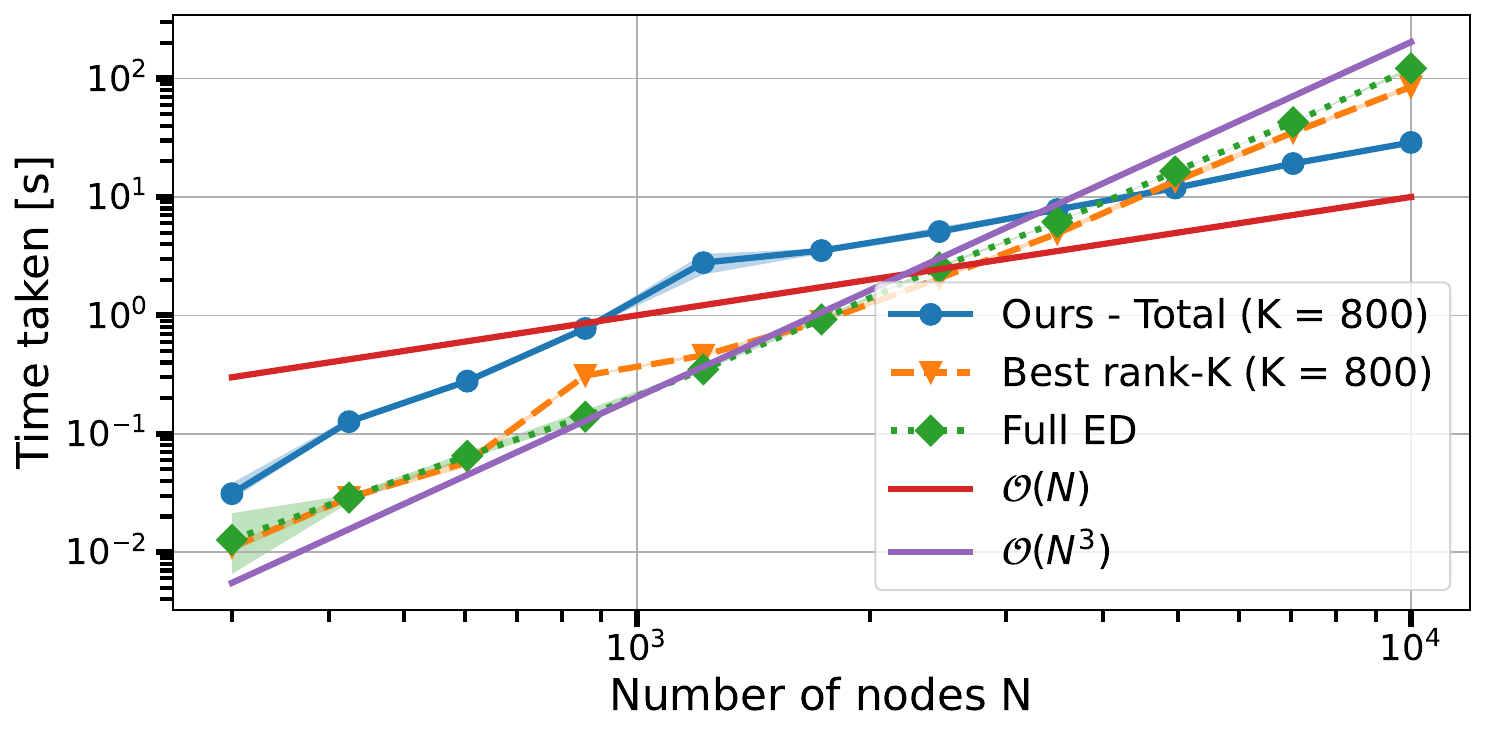}
    \end{subfigure}
    \caption{Swiss-Roll graph. Top left: average computation time as a function of $K$, for $N=5\,000$. Top right, bottom: average computation time as a function of $N$, for $K=800$. Shaded areas represent minimum and maximum values, over 5 trials.}
    \label{fig:Times}
\end{figure}

We observe that the time complexity experimentally matches that derived in Sec. \ref{sec:Complexity}, with the estimation of $\lambda_K$ independent of $K$ (except for $K \geq N$, where $\lambda_K$ is not needed), the filtering growing in $\mathcal{O}(K)$ for fixed $N$ (top-left) and in $\mathcal{O}(N)$ for fixed $K$ (top-right). The estimation of $\lambda_K$ seems to be growing only slightly faster than the predicted $\mathcal{O}(N \log N)$ (top-right), and the total time effectively grows in $\mathcal{O}(N)$ for fixed $K$ (and $K \leq N$). Our method takes slightly more time than exact decompositions for small $N$, but the complexity of those is in $\mathcal{O}(N^3)$ (as confirmed by our experiments --- bottom), which is why our method is faster for large $N$. These results indicate that we can hope to achieve significant computational gains for very large graphs.

\section{Conclusion}
We presented a new method for constructing random graph node embeddings that yield accurate low-rank approximations to specified graph kernels. The approach is especially effective for spectrally localized kernels, which are difficult to approximate with existing techniques. Its computational efficiency further enables kernel-based learning on large-scale graphs.

\appendix

\section{Proof of Proposition \ref{prop:main}:}
\label{appendix}
In this section we prove (\ref{eq:ERBound}) of Proposition \ref{prop:main}. Our arguments are largely inspired by the proof of \cite[Theorem 9.1]{halko2011finding}.
Before proving (\ref{eq:ERBound}), let us introduce some useful Lemmas directly originating from \cite{halko2011finding}.
\setcounter{theorem}{0}
\begin{lemma}[\hspace{-0.15cm} \cite{halko2011finding}, Proposition 8.1]
    \label{lem:ConjRule}
    Let $\bs M \succcurlyeq \bs 0$, and $\bs N \succcurlyeq M$. Then, for any real matrix $\bs A$,
    \begin{align*}
        \bs A^\top \bs M \bs A \succcurlyeq \bs 0 \quad \text{and} \quad \bs A^\top \bs N \bs A \succcurlyeq \bs A^\top \bs M \bs A.
    \end{align*}
\end{lemma}

\begin{lemma}[\hspace{-0.15cm} \cite{halko2011finding}, Proposition 8.2]
    \label{lem:InvPert}
    Let $\bs M \succcurlyeq \bs 0$. Then,
    \begin{align*}
        \bs I - (\bs I + \bs M)^{-1} \preccurlyeq \bs M
    \end{align*}
\end{lemma}

\begin{lemma}[\hspace{-0.15cm} \cite{halko2011finding}, Proposition 8.3]
    \label{lem:NormPart}
    Let $\bs M \succcurlyeq \bs 0$ be a real matrix partitioned as
    \begin{equation*}
        \bs M = \begin{bmatrix}
            \bs A & \bs B \\
            \bs B^\top & \bs C
        \end{bmatrix}.
    \end{equation*}
    Then, $$\| \bs M \| \leq \| \bs A \| + \| \bs C \|.$$
\end{lemma}

\begin{lemma}[\hspace{-0.15cm} \cite{halko2011finding}, Proposition 10.1]
    \label{lem:threeMat}
    Fix matrices $\bs S$, $\bs T$, and draw a standard Gaussian matrix $\bs G$ (\ie such that $G_{ij} \sim_\iid \mathcal{N}(0,1)$ for all $i,j$). Then,
    \begin{align*}
        \mathbb{E}_{\bs G} \|\bs{SGT} \| \le \| \bs S \| \|\bs T \|_F + \|\bs S\|_F \|\bs T\|.
    \end{align*}
\end{lemma}

\begin{lemma}[\hspace{-0.15cm} \cite{halko2011finding}, Proposition 10.2]
    \label{lem:GaussPI}
    Draw a $K \times (K+r)$ standard Gaussian matrix $\bs G$ with $K \ge 2$ and $r \ge 2$. Then,
    \begin{align*}
        (\mathbb{E} \|\bs G^\dagger\|^2_F )^{1/2}  = \sqrt{\frac{K}{r-1}}
    \end{align*}
    and 
    \begin{align*}
        \mathbb{E} \|\bs G^\dagger\| \le e \frac{\sqrt{K+r}}{r}.
    \end{align*}
\end{lemma}

\textit{Proof of (\ref{eq:ERBound}).} Let us first introduce some simplifying notations. We define $\bs \Lambda_1 = \bs \Lambda_{:K}$, $\bs V_1 = \bs V_{:K}$, $\bs V_2 = \bs V_{K:}$ with $\bs V = (\bs V_1, \bs V_2)$, $\bs \Lambda_2 =\diag(\lambda_{K+1}, \ldots, \lambda_N)$, $\bs G_1=\bs V_1^\top \bs G$, $\bs G_2=\bs V_2^\top \bs G$, and $\bs P_{\bs A}$ as the projector onto the column space of $\bs A$ (for an orthogonal matrix $\bs Q$, $\bs P_{\bs Q} =  \bs{QQ}^\top$). Note that because standard Gaussian matrices are invariant under rotations, $\bs V^\top \bs G$ is a standard Gaussian matrix, and $\bs G_1$ and $\bs G_2$ are also standard Gaussian matrices.
Let us further introduce
\begin{align}
\bftilde{\bs A}
&=
\bs V^\top
h^{\frac{1}{2}}(\bs L)
=
\begin{bmatrix}
h^{\frac{1}{2}}(\bs \Lambda_1) \bs V_1^\top \\
h^{\frac{1}{2}}(\bs \Lambda_2) \bs V_2^\top
\end{bmatrix}, \label{eq:ATdef}\\
\tilde{\bs B}
&=
\bs V^\top
p_{\chi}(\bs L)\bs G
=
\begin{bmatrix}
p_{\chi}(\bs \Lambda_1) \bs G_1 \\
p_{\chi}(\bs \Lambda_2) \bs G_2
\end{bmatrix}. \label{eq:BTdef}
\end{align}
Recall that $\bs Q = \ortho(p_{\chi}(\bs L) \bs G)$, such that $\ts \ortho(\tilde{\bs B}) = \bs V^\top \bs Q$. Owing to the unitary invariance of the spectral norm and to the fact that $\bs V^\top \bs P_{\bs Q} \bs V = \bs P_{\bs V^\top\bs Q}$, we have the identity
\begin{align}
\ts \cl E_R := \|(\bs I-\bs P_{\bs Q})h^{\frac{1}{2}}(\bs L)\|
&=
\|\bs V^\top(\bs I-\bs P_{\bs Q})\bs V\bftilde{\bs A}\| \nonumber \\
&=
\|(\bs I-\bs P_{\bs V^\top\bs Q})\bftilde{\bs A}\| \nonumber \\
&=
\|(\bs I-\bs P_{\tilde{\bs B}})\bftilde{\bs A}\|. \label{eq:AuxEquiv}
\end{align}

In view of~\eqref{eq:AuxEquiv}, we can thus execute the proof for the auxiliary matrix $\tilde{\bs A}$, and the auxiliary projector $P_{\tilde{\bs B}}$.

\medskip

The main argument of the proof is based on matrix perturbation theory. We start with a matrix related to $\tilde{\bs B}$:
\begin{equation*}
\bs W
=
\begin{bmatrix}
p_{\chi}(\bs \Lambda_1) \bs G_1 \\
\bs 0
\end{bmatrix}.
\end{equation*}
The matrix $\bs W$ has the same range as a related matrix formed by ``flattening out'' the spectrum of the top block. Indeed, since (with probability 1) $p_{\chi}(\bs \Lambda_1) \bs G_1$ has full row rank,
\begin{equation*}
\operatorname{span}(\bs W)
=
\operatorname{span}
\begin{bmatrix}
\bs I \\
\bs 0
\end{bmatrix}.
\end{equation*}
This immediately yields
\begin{equation}
\label{eq:Wproj}
\bs P_{\bs W}
=
\begin{bmatrix}
\bs I & \bs 0 \\
\bs 0 & \bs 0
\end{bmatrix},
\qquad
\bs I-\bs P_{\bs W}
=
\begin{bmatrix}
\bs 0 & \bs 0 \\
\bs 0 & \bs I
\end{bmatrix}.
\end{equation}

In other words, the range of $\bs W$ aligns with the first $K$ coordinates, which span the same subspace as the first $K$ left singular vectors of $\bftilde{\bs A}$. Therefore, $\operatorname{span}(\bs W)$ captures the action of $\bftilde{\bs A}$, which is what we wanted from $\operatorname{span}(\tilde{\bs B})$.

We treat the matrix $\tilde{\bs B}$ as a perturbation of $\bs W$, and hope that their ranges are close to each other. Note that $\tilde{\bs B} \approx \bs W$ as $p_\chi(\bs \Lambda_2) \approx \bs 0$. To make the comparison rigorous, let us emulate the arguments outlined in the last paragraph. Referring to (\ref{eq:BTdef}), we flatten out the top block of $\tilde{\bs B}$ to obtain the matrix
\begin{equation}
\bs Z
=
\tilde{\bs B}\,
\bs G_1^\dagger
p_{\chi}(\bs \Lambda_1)^{-1}
=
\begin{bmatrix}
\bs I \\
\bs F
\end{bmatrix}, \label{eq:Zdef}
\end{equation}
where
\begin{equation*}
\bs F = p_{\chi}(\bs \Lambda_2) \bs G_2 \bs G_1^\dagger p_{\chi}(\bs \Lambda_1)^{-1}.
\end{equation*}

The construction~\eqref{eq:Zdef} ensures that $\operatorname{span}(\bs Z)\subseteq \operatorname{span}(\tilde{\bs B})$, therefore the error satisfies
\begin{equation*}
\|(\bs I-\bs P_{\tilde{\bs B}})\bftilde{\bs A}\|
\le
\|(\bs I-\bs P_{\bs Z})\bftilde{\bs A}\|.
\end{equation*}
Squaring this relation, we obtain
\begin{align}
\|(\bs I-\bs P_{\tilde{\bs B}})\bftilde{\bs A}\|^2
&\le
\|(\bs I-\bs P_{\bs Z})\bftilde{\bs A}\|^2 \nonumber \\
&=
\|\bftilde{\bs A}^\top(\bs I-\bs P_{\bs Z})\bftilde{\bs A}\| \nonumber \\
&=
\|h^{\frac{1}{2}}(\bs \Lambda)(\bs I-\bs P_{\bs Z})h^{\frac{1}{2}}(\bs \Lambda)\|. \label{eq:2sidedProj}
\end{align}
The last identity follows from (\ref{eq:ATdef}) ($\bftilde{\bs A}=h^{\frac{1}{2}}(\bs \Lambda)\bs V^\top$) and from the unitary invariance of the spectral norm.

To continue, we need a detailed representation of the projector $\bs I-\bs P_{\bs Z}$. The construction~\eqref{eq:Zdef} ensures that $\bs Z$ has full column rank, so we can express $\bs P_{\bs Z}$ explicitly as
\begin{align*}
\bs P_{\bs Z}
=
\bs Z(\bs Z^\top\bs Z)^{-1}\bs Z^\top
=
\begin{bmatrix}
\bs I \\
\bs F
\end{bmatrix}
(\bs I+\bs F^\top\bs F)^{-1}
\begin{bmatrix}
\bs I & \bs F^\top
\end{bmatrix}.
\end{align*}
Expanding this expression, we find that the complementary projector satisfies
\begin{equation*}
\bs I-\bs P_{\bs Z}
=
\begin{bmatrix}
\bs I-(\bs I+\bs F^\top\bs F)^{-1}
&
-(\bs I+\bs F^\top\bs F)^{-1}\bs F^\top
\\[0.5ex]
-\bs F(\bs I+\bs F^\top\bs F)^{-1}
&
\bs I-\bs F(\bs I+\bs F^\top\bs F)^{-1}\bs F^\top
\end{bmatrix}.
\end{equation*}

The partitioning here conforms with the partitioning of $\bs \Lambda$. This block matrix expression is less fearsome than it looks. Lemma \ref{lem:InvPert} shows that the top-left block verifies
\[
\bs I-(\bs I+\bs F^\top\bs F)^{-1} \preccurlyeq \bs F^\top\bs F.
\]
The bottom-right block satisfies
\[
\bs I-\bs F(\bs I+\bs F^\top\bs F)^{-1}\bs F^\top \preccurlyeq \bs I
\]
because $\bs F(\bs I+\bs F^\top\bs F)^{-1}\bs F^\top \succcurlyeq \bs 0$. We abbreviate the off-diagonal blocks with the symbol $\bs C=-(\bs I+\bs F^\top\bs F)^{-1}\bs F^\top$. In summary,
\[
\bs I-\bs P_{\bs Z}
\preccurlyeq
\begin{bmatrix}
\bs F^\top\bs F & \bs C \\
\bs C^\top & \bs I
\end{bmatrix}.
\]

This relation exposes the key structural properties of the projector. Compare this relation with the expression~\eqref{eq:Wproj} for the ``ideal'' projector $\bs I-\bs P_{\bs W}$.

We conjugate the last relation by $h^{\frac{1}{2}}(\bs \Lambda)$, and invoke Lemma~\ref{lem:ConjRule} to obtain
\begin{multline*}
h^{\frac{1}{2}}(\bs \Lambda)(\bs I-\bs P_{\bs Z})h^{\frac{1}{2}}(\bs \Lambda)
\preccurlyeq \\
\quad \begin{bmatrix}
h^{\frac{1}{2}}(\bs \Lambda_1)\bs F^\top\bs Fh^{\frac{1}{2}}(\bs \Lambda_1)
&
h^{\frac{1}{2}}(\bs \Lambda_1)\bs C h^{\frac{1}{2}}(\bs \Lambda_2)
\\[0.5ex]
h^{\frac{1}{2}}(\bs \Lambda_2)\bs C^\top h^{\frac{1}{2}}(\bs \Lambda_1)
&
h^{\frac{1}{2}}(\bs \Lambda_2)h^{\frac{1}{2}}(\bs \Lambda_2)
\end{bmatrix}.
\end{multline*}
Lemma \ref{lem:ConjRule} demonstrates that the matrix on the left-hand side is psd, so the matrix on the right-hand side is also psd. Lemma~\ref{lem:NormPart} results in the norm bound
\begin{align*}
\|h^{\frac{1}{2}}(\bs \Lambda)(\bs I-\bs P_{\bs Z})h^{\frac{1}{2}}(\bs \Lambda)\|
&\le
\|h^{\frac{1}{2}}(\bs \Lambda_1)\bs F^\top\bs Fh^{\frac{1}{2}}(\bs \Lambda_1)\| \\
& \qquad \quad + \|h^{\frac{1}{2}}(\bs \Lambda_2)h^{\frac{1}{2}}(\bs \Lambda_2)\| \\
&=
\|\bs Fh^{\frac{1}{2}}(\bs \Lambda_1)\|^2
+
\|h^{\frac{1}{2}}(\bs \Lambda_2)\|^2.
\end{align*}
Recall that $\bs F = p_{\chi}(\bs \Lambda_2) \bs G_2 \bs G_1^\dagger p_{\chi}(\bs \Lambda_1)^{-1}$, so
\begin{multline*}
\|h^{\frac{1}{2}}(\bs \Lambda)(\bs I-\bs P_{\bs Z})h^{\frac{1}{2}}(\bs \Lambda)\|
\le \\
\quad \|p_{\chi}(\bs \Lambda_2)\bs G_2\bs G_1^\dagger p_{\chi}(\bs \Lambda_1)^{-1} h^{\frac{1}{2}}(\bs \Lambda_1)\|^2 + \|h^{\frac{1}{2}}(\bs \Lambda_2)\|^2.
\end{multline*}
Inserting this into (\ref{eq:2sidedProj}) and using (\ref{eq:AuxEquiv}), we get
\begin{align*}
    \ts \cl E_R^2 &\leq \| h^{\frac{1}{2}}(\bs \Lambda_2) \|^2 + \| p_{\chi}(\bs \Lambda_2) \bs G_2 \bs G_1^\dagger p_{\chi}(\bs \Lambda_1)^{-1} h^{\frac{1}{2}}(\bs \Lambda_1)\|^2
\end{align*}
Now, Assumption \ref{ass:pchi} yields $\|p_{\chi}(\bs \Lambda_1)^{-1}\| \leq 2$, such that
\begin{align}
    \ts \cl E_R &\leq \| h^{\frac{1}{2}}(\bs \Lambda_2) \| + \| p_{\chi}(\bs \Lambda_2) \bs G_2 \bs G_1^\dagger \| \|p_{\chi}(\bs \Lambda_1)^{-1} \| \|h^{\frac{1}{2}}(\bs \Lambda_1)\| \nonumber \\
    &\leq \| h^{\frac{1}{2}}(\bs \Lambda_2) \| + 2 \| p_{\chi}(\bs \Lambda_2) \bs G_2 \bs G_1^\dagger \| \label{eq:FixedBound}
\end{align}
We are now ready to use Lemma \ref{lem:threeMat} and Lemma \ref{lem:GaussPI}, which yield
\begin{align*}
    &\ts \bb E (\| p_{\chi}(\bs \Lambda_2) \bs G_2 \bs G_1^\dagger \|) \nonumber \\
    &\ts \quad = \bb E_{\bs G_1} \bb E_{\bs G_2} (\| p_{\chi}(\bs \Lambda_2) \bs G_2 \bs G_1^\dagger \|) \\
    &\ts \quad \leq \bb E_{\bs G_1} (\| p_{\chi}(\bs \Lambda_2) \| \| \bs G_1^\dagger \|_F + \| p_{\chi}(\bs \Lambda_2) \|_F \| \bs G_1^\dagger \|) \\
    &\ts \quad \leq \| p_{\chi}(\bs \Lambda_2) \| \bb E_{\bs G_1} (\| \bs G_1^\dagger\|^2_F)^{1/2} + \| p_{\chi}(\bs \Lambda_2) \|_F\bb E_{\bs G_1} (\| \bs G_1^\dagger \|)  \\
    &\ts \quad \leq \sqrt{\frac{K}{r-1}} \| p_{\chi}(\bs \Lambda_2) \| +  e \frac{\sqrt{K + r}}{r} \| p_{\chi}(\bs \Lambda_2) \|_F.
\end{align*}
Inserting back into (\ref{eq:FixedBound}) concludes the proof. \qed

\end{document}